\documentclass[letterpaper, 10 pt, conference]{ieeeconf}  

\usepackage[hidelinks]{hyperref}
\hypersetup{
    colorlinks=true,
    linkcolor=black,
    filecolor=magenta,      
    urlcolor=blue,
}
\usepackage{subfigure}
\usepackage[noadjust]{cite}
\usepackage{amsmath}
\usepackage{graphicx}
\usepackage{amssymb}
\usepackage{colortbl}
\usepackage{arydshln}
\usepackage{stfloats}
\usepackage{lipsum}
\usepackage[lined,boxed,ruled,linesnumbered]{algorithm2e}
\SetKwInOut{KwIn}{input}
\SetKwInOut{KwOut}{output}
\usepackage{multirow}
\usepackage{extarrows}
\usepackage{mathrsfs}
\usepackage{upgreek}
\newtheorem{theorem}{Theorem}
\newtheorem{remark}{Remark}

\newtheorem{problem}{Problem}

\newtheorem{assumption}{Assumption}

\newtheorem{definition}{Definition}
\newtheorem{lemma}{Lemma}

\usepackage[table]{xcolor}

\IEEEoverridecommandlockouts                              
\title{\LARGE \bf Compression Repair for Feedforward Neural Networks Based on Model Equivalence Evaluation}
\author{Zihao Mo, Yejiang Yang, Shuaizheng Lu, and Weiming Xiang
\thanks{This research was supported by the National Science Foundation, under NSF CAREER Award no. 2143351, NSF CNS Award no. 2223035, and NSF IIS Award no. 2331938.}
\thanks{The authors are with the School of Computer and Cyber Sciences, Augusta University, Augusta, GA 30912, USA. {\tt\small wxiang@augusta.edu}}
}

\begin{document}
\maketitle
\thispagestyle{empty}
\pagestyle{empty}

\begin{abstract}
In this paper, we propose a method of repairing compressed Feedforward Neural Networks (FNNs) based on equivalence evaluation of two neural networks. In the repairing framework, a novel neural network equivalence evaluation method is developed to compute the output discrepancy between two neural networks. The output discrepancy can quantitatively characterize the output difference produced by compression procedures. Based on the computed output discrepancy, the repairing method first initializes a new training set for the compressed networks to narrow down the discrepancy between the two neural networks and improve the performance of the compressed network. Then, we repair the compressed FNN by re-training based on the training set. We apply our developed method to the MNIST dataset to demonstrate the effectiveness and advantages of our proposed repair method.
\end{abstract}

\section{Introduction}
Back in 1943, McCulloch and Pitts \cite{mcculloch1943logical} brought up an idea about using logical calculus to simulate nervous activity, recognized as the origin of neural networks. Since then, neural networks have developed over the decades and become a fundamental tool in modern intelligent society. It has been applied in many areas, such as pattern recognition \cite{554195}, \cite{LITJENS201760}, image processing \cite{NIPS2012_c399862d}, \cite{SCHMIDHUBER201585}, computer vision \cite{khan2018guide}, \cite{wang2015visual}, etc. However, the evolution of neural networks is accompanied by the exponential growth of the scale and computation cost. According to the paper \cite{8970294}, \cite{thompson2022computational}, training a large feedforward neural network is power-consuming and is in high demand for memory usage. In an ACAS Xu \cite{owen2019acas} verification problem, 45 Feedforward Neural Networks (FNNs) have been deployed, with a total of 300 neurons of each network, which requires huge computational resources and is time-consuming \cite{katz2017reluplex}. Thus, the compression of neural networks becomes a hot topic in the industrial area, which can shrink down the scale of FNNs to be deployed in practical applications. Two main state-of-the-art compression methods are pruning \cite{Yang_2017_CVPR} and quantization \cite{zhou2017incremental}, where the former aims to reduce neurons and layers of the model, and the latter focuses on replacing high-precision parameters with low-precision parameters. 

However, neural network compression always comes at a cost. In the survey \cite{8995969}, the authors outlined four methods for compression and acceleration, but the summary also identified certain potential issues, notably a substantial reduction in the accuracy of the compressed network. In \cite{xiang2018reachability}, the severness of accuracy loss and importance of safety verifications in Cyber-Physical Systems (CPS) applications has been addressed; for example, collisions may occur if ACAS Xu verification fails. Thus, the verification of compressed FNNs is essential before deploying them. The paper \cite{xiang2022approximate} gives a concrete value to characterize the difference between two FNNs by performing the reachability analysis between the networks, which is more intuitive to identify whether the network meets those criteria. This hybrid zonotopes method \cite{zhang2023reachability} is also based on reachability analysis to verify the safety robustness of a neural feedback system by providing a quantitative result.

In this paper, we propose a novel merging method to perform reachability analysis between two feedforward neural networks with the same inputs to evaluate the equivalence of the compressed network with respect to the original one. With the given discrepancy result, we can identify the guaranteed output reachable domain of the compressed network. Further, we propose a repair framework for the compressed FNN based on the equivalence result to narrow down the discrepancy with the original FNN while retaining its performance in solving specific tasks. To demonstrate the effectiveness of our repair method, we apply it to classify the MNIST database and compare the repair outcomes of the compressed FNN with respect to the original FNN.

The remainder of the paper is organized as follows: Section II is about Preliminaries. Section III introduces our merging method and the framework of the repair method. Section IV is the experiment demonstrating the repair method. Section V presents the conclusion.

\section{Preliminaries}
In this paper, an FNN $\Phi: \mathbb{R}^{n_{0}} \rightarrow \mathbb{R}^{n_{L}}$ is defined in the recursive equations form of
\begin{equation}
\label{eq: FNN def}
\begin{cases} 
    \mathbf{y}^{\{l\}} = \phi^{\{l\}}(\mathbf{y}^{\{l-1\}}), \ l = 1,\ldots, L    \\
     \mathbf{y}^{\{L\}} =  \Phi(\mathbf{x}^{\{0\}}),~\mathrm{where} ~\mathbf{x}^{\{0\}}=\mathbf{y}^{\{0\}}
\end{cases}
\end{equation}
where $\mathbf{y}^{\{l\}}$ is the output of the $l$th layer, $\mathbf{x}^{\{0\}} \in \mathbb{R}^{n^{\{0\}}}$ is the input and $\mathbf{y}^{\{L\}}\in \mathbb{R}^{n^{\{l\}}}$ is the output of the FNN, respectively.
$\phi^{\{l\}}$ denotes the operation of the $l$th layer of the FNN, which can be fully connected layer $\phi_{\mathrm{fc}}$ or ReLU layer $\phi_{\mathrm{ReLU}}$, and $\mathbf{y}^{\{l\}}$ is the output of the $l$th layer. The fully connected operation $\phi_{\mathrm{fc}}$ is defined as
\begin{equation} \label{eq:fc_output}
    \mathbf{y}^{\{l\}}=\phi_{\mathrm{fc}}(\mathbf{y}^{\{l-1\}})=\mathbf{W}^{\{l\}}\mathbf{y}^{\{l-1\}}+\mathbf{b}^{\{l\}}
\end{equation}
where $\mathbf{W}^{\{l\}} \in \mathbb{R}^{n^{\{l\}}\times n^{\{l-1\}}}$ and $\mathbf{b}^{\{l\}} \in \mathbb{R}^{n^{\{l\}}}$ denote the weight matrices and the bias vectors for layer $l$, respectively. The ReLU operation $\phi_{\mathrm{ReLU}}$ is defined as
\begin{equation}
\label{eq: relu function}
    \phi_{\mathrm{ReLU}}(\mathbf{y}^{\{l\}})=[\max(0,y_1^{\{l\}}),\ldots,\max(0,y_n^{\{l\}})]^{T}
\end{equation}
where $y_i^{\{l\}}$ is the $i$th element of the vector $\mathbf{y}^{\{l\}}$ in (\ref{eq:fc_output}).

To enable sound equivalence evaluation for two FNNs, which essentially needs to consider all possible outputs of the networks, the following reachable set of FNNs is introduced. 

\begin{definition}
    Given an input set $\mathcal{X}^{\{0\}}\in \mathbb{R}^{n^{\{0\}}}$ for FNN (\ref{eq: FNN def}), we define the following set
    \begin{equation}
    \label{eq: set def}
        \mathcal{Y}^{\{L\}} = \{\mathbf{y}^{\{L\}}\mid \mathbf{y}^{\{L\}}=\Phi(\mathbf{x}^{\{0\}}),~ \mathbf{x}^{\{0\}}\in \mathcal{X}^{\{0\}}\}
    \end{equation}
where $\mathcal{Y}^{\{L\}} \subseteq \mathbb{R}^{n^{\{L\}}}$ is called the output set of FNN (\ref{eq: FNN def}).
\end{definition}
\begin{remark}
There are a number of available reachable set representations used for FNN reachability analysis, such as zonotope \cite{singh2018fast}, polytope \cite{singh2019abstract}, FVIM \cite{henk2017basic}, etc. For instance, in the MNIST dataset application Section IV, we will use ImageStar proposed in \cite{tran2020verification} in our approach. ImageStar $\Theta$ is a tuple $\langle c, V, P \rangle$ where $c\in \mathbb{R}^{h\times w\times nc}$ is the anchor image, $V=\{v_1,v_2,\cdots,v_m\}$ is a set of $m$ images in $\mathbb{R}^{h\times w\times nc}$ called generator images, $P:\mathbb{R}^m\leftarrow \{\top, \bot \}$ is a predicate, and $h, w, nc$ are the height, width, and number of channels of the images, respectively. The generator images are arranged to form the ImageStar's $h\times w\times nc\times m$ basis array. The set of images represented by the ImageStar is:
\begin{equation}
    \Theta = \{x|x=c+\sum^{m}_{i=1}(\alpha_i v_i),~ \mathrm{where}~P(\alpha_i,\cdots ,\alpha_m)=\top \}, \nonumber
\end{equation}
in which we restrict the predicates to be a conjunction of linear constraints, $P(\alpha)\triangleq C\alpha \leq d$ where, for $p$ linear constraints, $C\in \mathbb{R}^{p\times m}$, $\alpha$ is the vector of $m$ variables, i.e., $\alpha=[\alpha_1, \cdots, \alpha_m]^T$, and $d\in \mathbb{R}^{p\times 1}$. An ImageStar is an empty set if and only if $P(\alpha)$ is empty.
\end{remark}

Based on the output reachable set defined in (\ref{eq: set def}), we can define the set value representation of an FNN as below
\begin{equation}
    \begin{cases}
    \mathcal{Y}^{\{l\}} = \phi^{\{l\}}(\mathcal{Y}^{\{l-1\}}), ~ l=1,2, \ldots, L \\
    \mathcal{Y}^{\{L\}} = \Phi(\mathcal{X}^{\{0\}}) ,~\mathrm{where}~\mathcal{X}^{\{0\}}=\mathcal{Y}^{\{0\}} 
    \end{cases}
\end{equation}
where $\mathcal{Y}^{\{l\}}$ denotes the output reachable set of $l$th layer, and, in particular, $\mathcal{Y}^{\{0\}}$ is the input set $\mathcal{X}^{\{0\}}$ and $\mathcal{Y}^{\{L\}}$ is the output set of the network.

In this paper, the equivalence evaluation aims to characterize the discrepancy between two FNNs, $\Phi_1$ and $\Phi_2$ under the following assumptions.

\begin{assumption}\label{assumption_1}
    The following assumptions hold for two neural networks $\Phi_1$ and $\Phi_2$:
\begin{enumerate}
    \item[(i)] The number of inputs of two neural networks are the same, i.e., $n_1^{\{0\}}=n_2^{\{0\}}$;
    \item[(ii)] The number of outputs of two neural networks are the same, i.e., $n_1^{\{L\}}=n_2^{\{L\}}$;
    \item[(iii)] The number of layers of two neural networks is the same, i.e., $L_1=L_2=L$;
    \item[(iv)] For each layer $l$, two neural networks perform the same operation.
\end{enumerate}
\end{assumption}

\begin{remark} It has to be pointed out that a typical compressed neural network usually consists of a reduced number of layers compared to the original network, which fails to satisfy (iii) and (iv) in Assumption \ref{assumption_1}. However, we can always extend the compressed network by incorporating additional layers as detailed in \cite{xiang2022approximate}. These additional layers equipped with identity weights and zero biases are mandated to transmit information to subsequent layers without any alterations, but meet the requirements of (iii) and (iv). 
\end{remark}

\section{Main Results}

\subsection{Equivalence Evaluation for Two FNNs}

Given an input set $\mathcal{X}^{\{0\}}$ for two FFNs $\Phi_1$ and $\Phi_2$ under Assumption \ref{assumption_1}, the equivalence evaluation in this work is given by quantifying the maximal discrepancy of  $\mathbf{y}_1^{\{L\}}$ and $\mathbf{y}_2^{\{L\}}$, where are the outputs of $\Phi_1$ and $\Phi_2$, respectively. To enable equivalence evaluation of $\Phi_1$ and $\Phi_2$, our first goal is to construct the discrepancy of the outputs of two FNNs with the same inputs, i.e., 
\begin{align}
    \delta = \Phi_1(\mathbf{x}^{\{0\}}) - \Phi_2(\mathbf{x}^{\{0\}}),~\mathbf{x}^{\{0\}} \in \mathcal{X}^{\{0\}}
\end{align}
where $\delta \in\mathbb{R}^{n^{\{L\}}}$ is the discrepancy vector.


For fully connected layers $\phi_{\mathrm{fc}}$ and ReLU layers $\phi_{\mathrm{ReLU}}$, we can obtain the following two results. 

\begin{lemma}\label{lemma_1}
    Consider two FFNs $\mathcal{N}_1$ and $\mathcal{N}_2$ under Assumption \ref{assumption_1}, the following result holds for fully connected layers
    \begin{align} \label{eq:lemma_1_1}
    \begin{bmatrix}
        \mathbf{y}_{1}^{\{l\}}
        \\
        \mathbf{y}_{2}^{\{l\}}
    \end{bmatrix} = \phi_{\mathrm{fc}}\left(\begin{bmatrix}
        \mathbf{y}_{1}^{\{l-1\}}
        \\
        \mathbf{y}_{2}^{\{l-1\}}
    \end{bmatrix}\right) = \tilde{\mathbf{W}}^{\{l\}}\begin{bmatrix}
        \mathbf{y}_{1}^{\{l-1\}}
        \\
        \mathbf{y}_{2}^{\{l-1\}}
    \end{bmatrix}+\tilde{\mathbf{b}}^{\{l\}}
    \end{align}
where $\tilde{\mathbf{W}}^{\{l\}} = \mathrm{diag}\{\mathbf{W}_{1}^{\{l\}},\mathbf{W}_{2}^{\{l\}}\}$ and $\tilde{\mathbf{b}}^{\{l\}}  = [
    (\mathbf{b}_{1}^{\{l\}})^{T}
    ,
    (\mathbf{b}_{2}^{\{l\}})^{T}
]^{T}$ in which $\mathbf{W}_{1}^{\{l\}}$, $\mathbf{W}_{2}^{\{l\}}$, $\mathbf{b}_{1}^{\{l\}}$, $\mathbf{b}_{2}^{\{l\}}$ are the weights and biases of  $\mathcal{N}_1$ and $\mathcal{N}_2$ at layer $l$.
\end{lemma}
\begin{proof}
    The result can be obtained straightforwardly by the definition of the fully connected layer in the form of (\ref{eq:fc_output}) such as
    \begin{align*}
    \phi_{\mathrm{fc}}\left(\begin{bmatrix}
        \mathbf{y}_{1}^{\{l-1\}}
        \\
        \mathbf{y}_{2}^{\{l-1\}}
    \end{bmatrix}\right)
        &= \begin{bmatrix}
            \phi_{\mathrm{fc}}(\mathbf{y}_1^{\{l-1\}})
            \\
            \phi_{\mathrm{fc}}(\mathbf{y}_2^{\{l-1\}})
        \end{bmatrix}
        \\
           &=\begin{bmatrix}
            \mathbf{W}^{\{l\}}_1(\mathbf{y}_1^{\{l-1\}})+\mathbf{b}_1^{\{l\}}
            \\
            \mathbf{W}^{\{l\}}_2(\mathbf{y}_2^{\{l-1\}})+\mathbf{b}_2^{\{l\}}
        \end{bmatrix}
        \\
        &=\tilde{\mathbf{W}}^{\{l\}} \begin{bmatrix}
        \mathbf{y}_{1}^{\{l-1\}}
        \\
        \mathbf{y}_{2}^{\{l-1\}}
    \end{bmatrix}+\tilde{\mathbf{b}}^{\{l\}}.
    \end{align*}
The proof is complete.
\end{proof}

\begin{lemma}\label{lemma_2}
    Consider two FFNs $\mathcal{N}_1$ and $\mathcal{N}_2$ under Assumption \ref{assumption_1}, the following result 
      \begin{align} \label{eq:lemma_2_1}
    \begin{bmatrix}
        \mathbf{y}_{1}^{\{l\}}
        \\
        \mathbf{y}_{2}^{\{l\}}
    \end{bmatrix} = \phi_{\mathrm{ReLU}}\left(\begin{bmatrix}
        \mathbf{y}_{1}^{\{l-1\}}
        \\
        \mathbf{y}_{2}^{\{l-1\}}
    \end{bmatrix}\right) = \begin{bmatrix}
        \phi_{\mathrm{ReLU}}(\mathbf{y}_{1}^{\{l-1\}})
        \\
        \phi_{\mathrm{ReLU}}\mathbf{y}_{2}^{\{l-1\}})
    \end{bmatrix}
    \end{align}
holds for ReLU layers.
\end{lemma}
\begin{proof}
    As the ReLU operation is performed in an element-wise manner, the result can be obtained straightforwardly. The proof is complete. 
\end{proof}
\begin{remark}
    It should be noted that the ReLU function may split the reachable set into multiple ones based on the linear constraints. The ReLU function is possible to perform in a different way for the pixel if its bounded range is across the zero. To handle the different situations, our method will split the linear constraints into two groups to force the bounded range of the pixel to fall into one side, leading to $\mathcal{Y}=\mathcal{Y}_1\cup \mathcal{Y}_2 \cup \ldots \cup \mathcal{Y}_n$. With the split strategy, the reachable set may increase exponentially based on the linear constraints, which increases the computation time.
\end{remark}

Besides fully connected and ReLU layers,  we introduce a comparison layer to
construct the
output discrepancy $\mathbf{y}_{\mathrm{cmp}} = \mathbf{y}_1^{\{L\}} - \mathbf{y}_2^{\{L\}}$ to evaluate the equivalence on two FNNs. The comparison layer is considered as an extra layer receiving the output $\mathbf{y}_1^{\{L\}}$ and $\mathbf{y}_2^{\{L\}}$ of FNNs $\mathcal{N}_1$ and $\mathcal{N}_2$ in the form of 
\begin{align} \label{eq:cmp_layer}
    \mathbf{y}_{\mathrm{cmp}} = \phi_{\mathrm{cmp}} \left(\begin{bmatrix}
        \mathbf{y}_{1}^{\{L\}}
        \\
        \mathbf{y}_{2}^{\{L\}}
    \end{bmatrix}\right) = \mathbf{W}_{\mathrm{cmp}}\begin{bmatrix}
        \mathbf{y}_{1}^{\{L\}}
        \\
        \mathbf{y}_{2}^{\{L\}}
    \end{bmatrix} 
\end{align}
where $\mathbf{W}_{\mathrm{cmp}} = \begin{bmatrix}
    \mathbf{I} & -\mathbf{I}
\end{bmatrix}$.

\begin{theorem} \label{thm_1}
    Consider two FFNs $\Phi_1$ and $\Phi_2$ under Assumption \ref{assumption_1}, the following result 
    \begin{align}
        \mathbf{y}_{\mathrm{cmp}} = \Phi_1(\mathbf{x}^{\{0\}}) - \Phi_2(\mathbf{x}^{\{0\}})
    \end{align}
    holds for the output $\mathbf{y}_{\mathrm{cmp}}$ of comparison layer defined in (\ref{eq:cmp_layer}).
\end{theorem}
\begin{proof}
    By Lemmas \ref{lemma_1} and \ref{lemma_2}, and under Assumption \ref{assumption_1}, we have 
    \begin{align}
       \begin{bmatrix}
        \mathbf{y}_{1}^{\{l\}}
        \\
        \mathbf{y}_{2}^{\{l\}}
    \end{bmatrix} = \phi^{\{l\}}\left(\begin{bmatrix}
        \mathbf{y}_{1}^{\{l-1\}}
        \\
        \mathbf{y}_{2}^{\{l-1\}}
    \end{bmatrix}\right)  \ l = 1,\ldots, L 
    \end{align}
    where $\phi^{\{l\}}$ can be either $\phi_{\mathrm{fc}}$ or $\phi_{\mathrm{ReLU}}$. It is worth noting that Lemmas \ref{lemma_1} and \ref{lemma_2} also provide the computation procedures to compute $ \mathbf{y}_{1}^{\{l\}}$ and $ \mathbf{y}_{2}^{\{l\}}$ for each layer. Therefore, with the same input $\mathbf{x}^{\{0\}}$, it leads to 
\begin{align}
       \begin{bmatrix}
        \mathbf{y}_{1}^{\{L\}}
        \\
        \mathbf{y}_{2}^{\{L\}}
    \end{bmatrix} = \begin{bmatrix}
        \phi^{\{L\}}\circ\cdots\circ\phi^{\{1\}}(\mathbf{x}^{\{0\}})
        \\
        \phi^{\{L\}}\circ\cdots\circ\phi^{\{1\}}(\mathbf{x}^{\{0\}})
    \end{bmatrix} = \begin{bmatrix}
        \Phi_1(\mathbf{x}^{\{0\}})
        \\
        \Phi_2(\mathbf{x}^{\{0\}})
    \end{bmatrix} 
    \end{align}

Then, from (\ref{eq:cmp_layer}), one can obtain
\begin{align}
      \mathbf{y}_{\mathrm{cmp}} = \begin{bmatrix}
    \mathbf{I} & -\mathbf{I}
\end{bmatrix}\begin{bmatrix}
        \mathbf{y}_{1}^{\{L\}}
        \\
        \mathbf{y}_{2}^{\{L\}}
    \end{bmatrix} =\Phi_1(\mathbf{x}^{\{0\}}) - \Phi_2(\mathbf{x}^{\{0\}}). 
\end{align}
The proof is complete.     
\end{proof}

As shown in Theorem \ref{thm_1}, the output of the comparison layer, i.e., $\mathbf{y}_{\mathrm{cmp}}$, is the discrepancy vector measuring the output difference between two FFNs $\Phi_1$ and $\Phi_2$. With the help of this discrepancy vector and reachability analysis, we can formally characterize the equivalence of two FNNs. A merged $L+1$ layer FNN $\tilde \Phi$ out of $\Phi_1$ and $\Phi_2$ can constructed in the form of 
    \begin{equation}
\label{eq:merged_FNN}
\begin{cases} 
    \tilde{\mathbf{y}}^{\{l\}} = \phi^{\{l\}}(\tilde{\mathbf{y}}^{\{l-1\}}), \ l = 1,\ldots, L    \\
    \tilde{\mathbf{y}}^{\{L+1\}} = \phi_{\mathrm{cmp}}(\tilde{\mathbf{y}}^{\{L\}}) 
    \\
     \tilde{\mathbf{y}}^{\{L+1\}} =  \tilde{\Phi}(\mathbf{x}^{\{0\}}),~\mathrm{where}~\mathbf{x}^{\{0\}}=\mathbf{y}^{\{0\}}
\end{cases}
\end{equation}
in which fully connected and ReLU hidden layers from $1$ to $L$ are defined by (\ref{eq:lemma_1_1}) and (\ref{eq:lemma_2_1}) and the output layer $L+1$ is defined by (\ref{eq:cmp_layer}). By performing reachability analysis for merged FNN (\ref{eq:merged_FNN}), i.e., computing the output reachable set $\tilde{\mathcal{Y}}^{\{L+1\}}$ of merged FNN (\ref{eq:merged_FNN}), we can formally characterize equivalence between two FNNs $\Phi_1$ and $\Phi_2$. For instance, we can compute the maximal distance of the outputs of two FFNs in terms of 
    \begin{align} \label{eq:delta_1}
\delta_{\max} = \max_{\tilde{\mathbf{y}}^{\{L+1\}} \in \tilde{\mathcal{Y}}^{\{L+1\}}} \left\|\tilde{\mathbf{y}}^{\{L+1\}}\right\|
    \end{align}
In some scenarios, we might be interested in the discrepancies for each dimension, such as the image recognition application in Section IV. We can also make use of the reachable set $\tilde{\mathcal{Y}}^{\{L+1\}}$ to compute the vector of maximal magnitudes at each dimension of $\tilde{\mathcal{Y}}^{\{L+1\}}$, i.e., 
\begin{align}\label{eq:delta_2}
    \tilde{\delta}_{\max} = \max_{\tilde{\mathbf{y}}^{\{L+1\}} \in \tilde{\mathcal{Y}}^{\{L+1\}}} \tilde{\mathbf{y}}^{\{L+1\}}
\end{align}
where the $\max$ operator performs element-wisely on $\tilde{\mathbf{y}}^{\{L+1\}}$. 

\begin{remark}
To perform the efficient equivalence evaluation, the reachable set computation is essential. There exist a number of tools available. For instance, as in the NNV neural network reachability analysis tool, the reachable sets are in
the form of a collection of polyhedral sets \cite{tran2020nnv}, and in the
IGNNV tool, the output reachable set is a family of interval
sets \cite{xiang2018output,xiang2020reachable}. For those types of convex sets, the equivalence evaluation metrics in the description of maximal values can be obtained by testing 
a finite number of vertices in convex sets.   
\end{remark}

\subsection{FNN Compression Repair}
Given an FNN $\Phi_1$ and its compressed version $\Phi_2$,  the goal of repairing the compressed $\Phi_2$ is that the discrepancy between $\Phi_1$ and $\Phi_2$ should satisfy a set of prescribed conditions described by set $\mathcal{O}$, e.g., $\mathcal{O} = \{\tilde{\mathbf{y}}^{\{L+1\}} \mid \left\|\tilde{\mathbf{y}}^{\{L+1\}} \right\|\le d\}$ where $d>0$ is a prescribed threshold. 
The general compressed FNN repair problem can be stated as follows.
\begin{problem}
    Given an FNN $\Phi_1$ and its compressed version $\Phi_2$, an input set $\mathcal{X}^{\{0\}}$, and a prescribed repairing target set $\mathcal{O}$, how does one modify the compressed FNN $\Phi_2$ such that 
    \begin{align} \label{eq:problem_1}
        \tilde{\mathcal{Y}}^{\{L+1\}} \subseteq \mathcal{O}
    \end{align}
where $\tilde{\mathcal{Y}}^{\{L+1\}}$ is the output reachable set of merged $L+1$ layer FNN $\tilde\Phi$ in the form of (\ref{eq:merged_FNN}) that is constructed out of $\Phi_1$ and $\Phi_2$. 
\end{problem}


To address the FNN compression repair problem, normally, the goal of the repair is to minimize the discrepancy between FNNs $\Phi_1$ and $\Phi_2$. From the optimization perspective, the repair problem can be described as 
\begin{align}
\min_{\mathbf{W}_2^{\{l\}},\mathbf{b}_2^{\{l\}},~l=1,\ldots,L}\ell(\mathbf{y}_{1}^{\{L\}},\mathbf{y}_{2}^{\{L\}}) 
\end{align}
where $\ell(\cdot)$ is the loss function describing the discrepancy such as  (\ref{eq:delta_1}) and (\ref{eq:delta_2}). 

To modify $\mathbf{W}_2^{\{l\}}$ and $\mathbf{b}_2^{\{l\}}$ to repair the compressed FNN, a new dataset has to be created for retraining the compressed FNN. A straightforward way is to generate $N$ retraining data pair $\{\mathbf{x}_i^{\{0\}},\mathbf{y}_{i,2}^{\{L\}}\}$, $i=1,\ldots,N$, from FNN $\Phi_2$,  and replace the output samples $\mathbf{y}_{i,2}^{\{L\}}$ with the outputs of the original $\Phi_1$ , i.e., $\{\mathbf{x}_i^{\{0\}},\mathbf{y}_{i,1}^{\{L\}}\}$, $i=1,\ldots,N$, which completely eliminates the discrepancy in data set. 
Furthermore, the cost function $\ell(\cdot)$ for the retraining data set can be then written into the mean square loss function in the retraining process as follows: 
\begin{align}
    \ell(\mathbf{y}_{1}^{\{L\}},\mathbf{y}_{2}^{\{L\}}) = \frac{1}{N}\sum_{i=1}^{N}\left\|\mathbf{y}_{i,1}^{\{L\}}-\mathbf{y}_{i,2}^{\{L\}}\right\| .
\end{align}
With the above loss function, the FNN $\Phi_2$ training process can be viewed as a data-driven procedure to search for the optimal solution.  

However, this method intends to cause overfitting issues and significantly deteriorates network performance, such as accuracy. In this work, we turn to gradually reduce the discrepancy by updating the $\mathbf{y}_{i,2}^{\{L\}}$ in the following way
\begin{equation}
\label{eq: update}
    \hat{\mathbf{y}}_{i,2}^{\{L\}} = \mathbf{y}_{i,2}^{\{L\}}  + \frac{1}{\alpha} \tilde{\delta}_{max}
\end{equation}
where $\alpha \geq 1$ is a tuning parameter to control the step size to the target output, and $\tilde{\delta}_{max}$ is defined by (\ref{eq:delta_2}). Therefore, the data in retraining data is modified to $\{\mathbf{x}_i^{\{0\}},\hat{\mathbf{y}}_{i,2}^{\{L\}}\}$, and the loss function becomes 
\begin{align}\label{eq:loss}
    \ell(\hat{\mathbf{y}}_{2}^{\{L\}}, \mathbf{y}_{2}^{\{L\}}) = \frac{1}{N}\sum_{i=1}^{N}\left\|\hat{\mathbf{y}}_{i,2}^{\{L\}}-\mathbf{y}_{i,2}^{\{L\}}\right\| .
\end{align}

\begin{figure}
	\centering
\includegraphics[scale=.42]{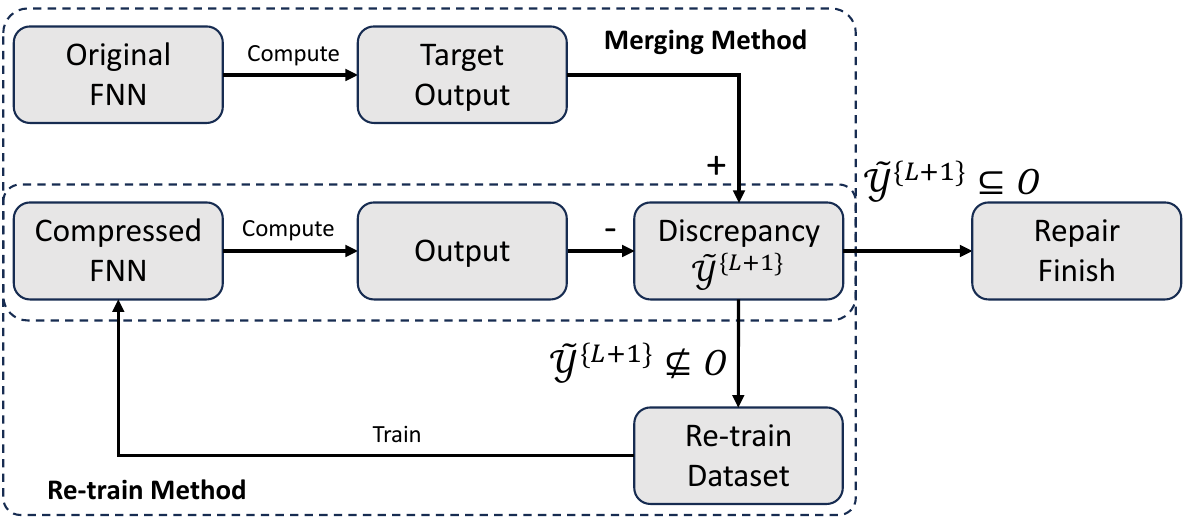}
	\caption{Framework of compressed feedforward neural network repair.}
	\label{FIG: Fig. 1}
\end{figure}
Thus, we can bring up the framework for compressed FNN repair based on equivalence evaluation as shown in Fig. \ref{FIG: Fig. 1}. 

\begin{itemize}
    \item \textbf{Initialization.} Given FNN $\Phi_1$ and its compression $\Phi_2$, we compute the output reachable set $\tilde{\mathcal{Y}}^{\{L+1\}}$ out of merged FNN. If (\ref{eq:problem_1}) is not satisfied, the compressed FNN needs to be re-trained. We compute the discrepancy $\tilde{\delta}_{max}$ based on $\tilde{\mathcal{Y}}^{\{L+1\}}$.    
    \item \textbf{Generate re-training data set.} Generate $N$ samples of $\{\mathbf{x}_i^{\{0\}},\mathbf{y}_{i,1}^{\{L\}},\mathbf{y}_{i,2}^{\{L\}}\}$, and build the re-training data set $\{\mathbf{x}_i^{\{0\}},\hat{\mathbf{y}}_{i,2}^{\{L\}}\}$, $i=1,\ldots,N$, based on (\ref{eq: update}). 
    \item \textbf{Re-train compressed FNN.} Modify the weights and bias of $\Phi_2$ by training $\Phi_2$ using data set $\{\mathbf{x}_i^{\{0\}},\hat{\mathbf{y}}_{i,2}^{\{L\}}\}$, $i=1,\ldots,N$, under the loss function (\ref{eq:loss}). 
    \item \textbf{Evaluate re-training outcome.} After re-training $\Phi_2$, we compute output reachable set $\tilde{\mathcal{Y}}^{\{L+1\}}$ and compares it with the target reachable domain. The repair process is finished only when (\ref{eq:problem_1}) is satisfied. Otherwise, it repeats the repairing process. 
\end{itemize}

Algorithm \ref{alg: repair} summarizes the repairing process for FNN compression based on the equivalence evaluation of two FNNs. It keeps re-training the compressed FNN until the discrepancy meets the requirement. A timeout counter is set up to avoid the repair process falling into a dead loop. The computation of discrepancy follows Theorem \ref{thm_1} to compute and update the discrepancy between the original FNN and the updated compressed FNN. After the repair process, the relationship between reachable set $\tilde{\mathcal{Y}}^{\{L+1\}}$ and repairing target set $\mathcal{O}$ is used to evaluate whether the repair process is a success. 
\begin{algorithm}[t!]
\caption{FNN Compression Repair}
\label{alg: repair}
\KwIn{Original FNN $\Phi_1$, 
    Compressed FNN $\Phi_2$, repairing target set $\mathcal{O}$}
\KwOut{Repaired Compressed FNN $\hat{\Phi}_2$}
\While{True}
{
Compute discrepancy $\tilde{\delta}_{max}$
\\
    Generate $\{\mathbf{x}_i^{\{0\}},\mathbf{y}_{i,1}^{\{L\}},\mathbf{y}_{i,2}^{\{L\}}\}$, $i=1,\ldots,N$ 
    \\
    $\hat{\mathbf{y}}_{i,2}^{\{L\}} \gets \mathbf{y}_{i,2}^{\{L\}}  + \frac{1}{\alpha} \tilde{\delta}_{max}$
   \\
    $\Phi_2 \gets \mathrm{retrain}(\Phi_2, \mathrm{Dataset}(\{\mathbf{x}_i^{\{0\}},\hat{\mathbf{y}}_{i,2}^{\{L\}}\})$ \\
    Compute reachable set $\tilde{\mathcal{Y}}^{\{L+1\}}$ for  $\Phi_2$  \\
    \If{$\tilde{\mathcal{Y}}^{\{L+1\}}\subseteq \mathcal{O}$ or timeout}
    {   
    $\hat{\Phi}_2 \gets {\Phi}_2$
    \\
        \textbf{break}  
    }
}
\textbf{return} $\hat{\Phi}_2$
\end{algorithm}

\section{Application to Compressed Feedforward Neural Networks Repairment}
In this section, to validate the effectiveness of our proposed approach, we use the MNIST data set to perform our task. We apply our equivalence evaluation method on the two FNNs and the repair method to the compressed FNN to narrow down the discrepancy\footnote{The code for the experiment is available at \url{github.com/aicpslab/FNN-repair}}. 

\subsection{Database}
MNIST \cite{lecun-mnisthandwrittendigit-2010} database contains a large number of handwritten digits and is famous as an image classification problem benchmark. The database contains 60,000 training images and 10,000 testing images. Each image is a $28\times28\times1$ grayscale image, and all images are classified into ten labels, from 0 to 9. 

\subsection{Experiment Set Up}
We train an FNN with three layers: the first has 256 neurons, the second has 64 neurons, and the third has 10 neurons. Each hidden layer is followed by a ReLU activation function. The FNN is trained with the training dataset of MNIST for five epochs. As for the compressed FNN, we apply the quantization aware training (QAT) \cite{krishnamoorthi2018quantizing} method to the original FNN, shrinking down the parameter size of the network. Table \ref{tab: overview} shows the comparison of the original network and the compressed network. Table \ref{tab: random im} shows the results with ten randomly picked images. The discrepancy is the mean of the maximum distance between the output value of the two FNNs among all label scores. 

To repair the compressed FNN, we set up the target reachable domains $\mathcal{O}$ as two-thirds of the original discrepancy domains. 
Every iteration compares the last discrepancy domain with the target domains. If it is not within the desired area, a re-train dataset for compressed FNN is generated based on the last discrepancy domain to re-retrain for three epochs. According to (\ref{eq: update}), we use different $\alpha=2,5,10,20$. The re-train process will time out after ten iterations. Except for the ten randomly chosen images, we also randomly chose ten re-train samples from the images where the original FNN gives the wrong predictions.  

\subsection{Results}
First, we perform the FNN repair with $\alpha=10$ in (\ref{eq: update}). Table \ref{tab: overview} is the overview of the three FNNs in the experiment. The original FNN has a $98\%$ accuracy but drops to $91\%$ after compression. However, compared to the size of the three FNNs, the compressed version's size is only one-fourth of the original one, which proves that compression helps decrease the scale of the neural network. Comparing the compressed FNN before repair and after repair, the performance of the FNN improves from $91\%$ to $98\%$ and almost achieves the same level as the original FNN, which our repairing method is able to improve the capability of the compressed network. In addition, Table \ref{tab: random im} shows the discrepancy result of our repair method. The total mean discrepancies after repair are all lower than our target values for each testing input, and some are even only one-tenth of the original value, which demonstrates the effectiveness of our repair method while keeping the performance. Fig. \ref{fig:random im} shows the repairing result in detail via a randomly chosen image. The input image is a handwritten digit ``9". The blue dots are the scoring output of the original network, with the highest score at label 9. The green whisker bar line on each label represents the guaranteed output range of the compressed network before repair. Relatively, the red whisker bar line is the guaranteed output range of the network after repair. It is obvious that each red whisker line is closer to the blue dot than the green one, proving that the discrepancy is mitigated after repair.

\begin{table}[t!]
    \centering
    \caption{Networks overview}
    \label{tab: overview}
    \begin{tabular}{cccc}
    \hline Network & Parameters & Size (KB) & Accuracy (\%) \\
    \hline Original Network & 218,058 & 855 & 98\%\\
    Compressed Network & 218,058 & 226 & 91\%\\
    Repaired Network & 218,058 & 226 & 98\%\\
    \hline
    \end{tabular}
\end{table}

\begin{table}[ht!]
    \centering
    \caption{Discrepancy results between before and after repair with random images}
    \begin{tabular}{ccc}
    \hline Input Image &  Mean $\delta$ (Before) & Mean $\delta$ (After) \\
    \hline 0 & 0.4750 & 0.0625 \\
    1 & 0.4189 & 0.0969 \\
    2 & 0.6339 & 0.1034 \\
    3 & 0.4746 & 0.0870 \\
    4 & 0.8272 & 0.1217 \\
    5 & 0.6200 & 0.1189 \\
    6 & 0.6458 & 0.0701 \\
    7 & 0.5440 & 0.1505 \\ 
    8 & 0.5908 & 0.0953 \\
    9 & 0.8283 & 0.1053 \\
    \hline
    \end{tabular}
    \label{tab: random im}
\end{table}

\begin{figure}[t!]
    \centering
    \includegraphics[scale=0.28]{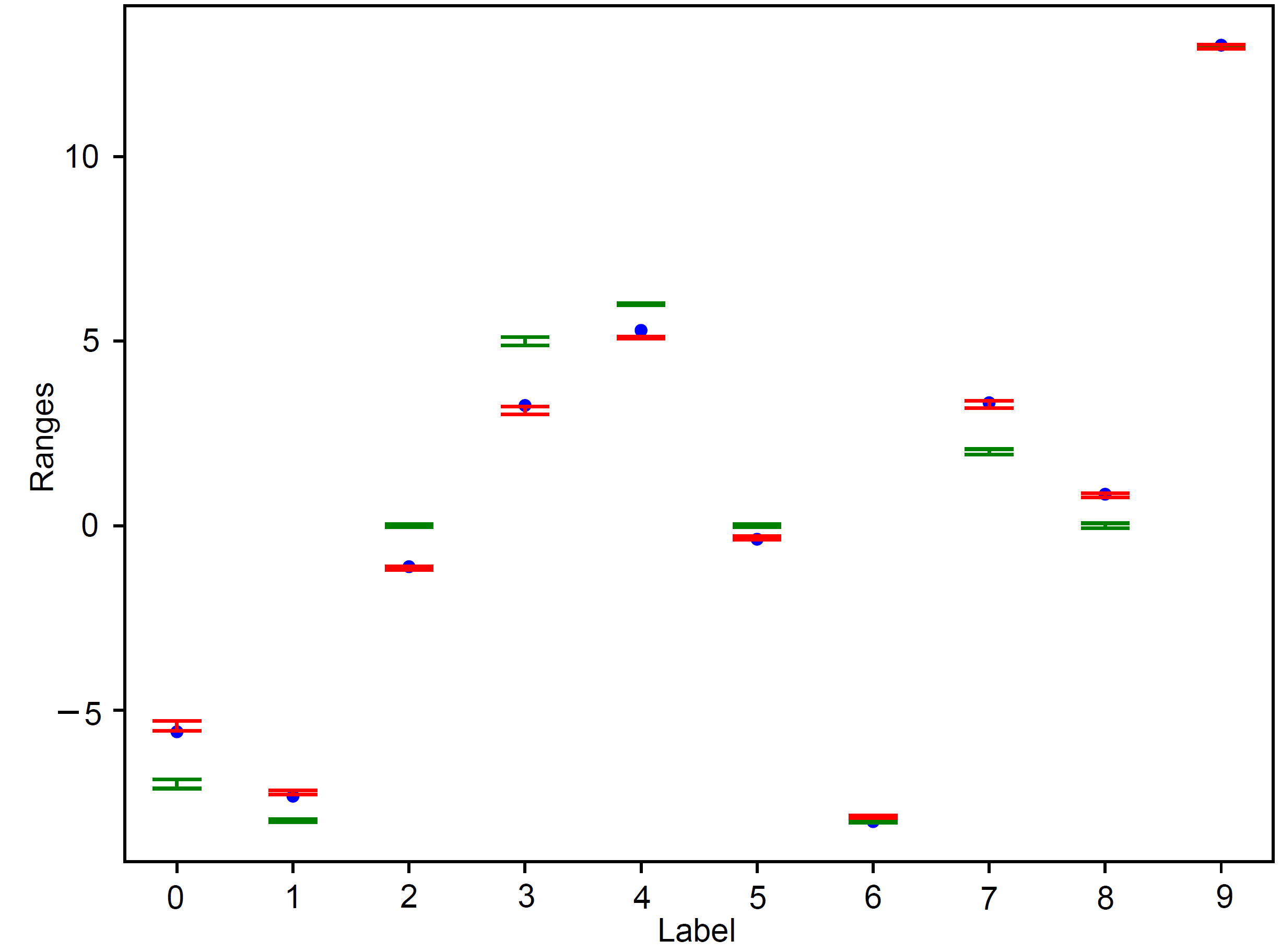}
    \caption{Repair results with a handwritten digit ``9". Blue dots are the output for the original network $\Phi_1$. The green whisker line represents the output range of the compressed network $\Phi_2$ before repair. The red whisker line represents the output range of the compressed network $\hat\Phi_2$ after repair. The outcome shows that the repaired network generates a more precise output range (red whisker lines) closer to the original outputs (blue dots).}
    \label{fig:random im}
\end{figure}
\begin{figure}[t!]
    \centering
    \includegraphics[scale=0.53]{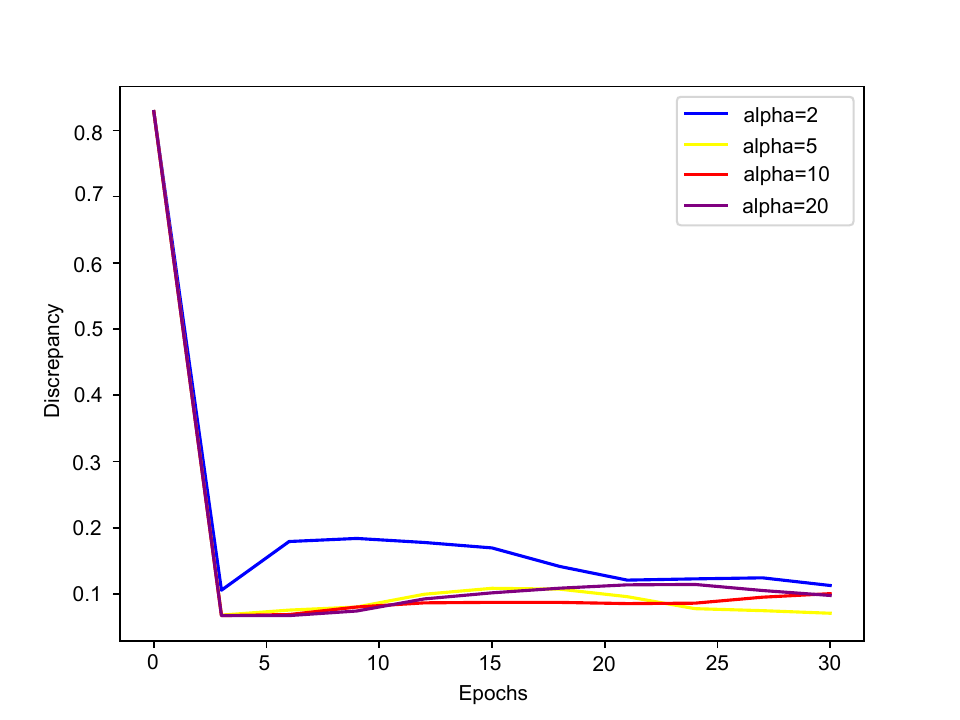}
    \caption{Repair result with a handwritten digit ``9". Different color lines are the average discrepancy of input images between the original network $\Phi_1$ and compressed network $\Phi_2$ with different $\alpha$ settings. The outcome shows that different $\alpha$ may lead to different repair performance, but the repair process can always decrease the discrepancy.}
    \label{fig: dis im}
\end{figure}

\begin{figure}[t!]
    \centering
    \includegraphics[scale=0.53]{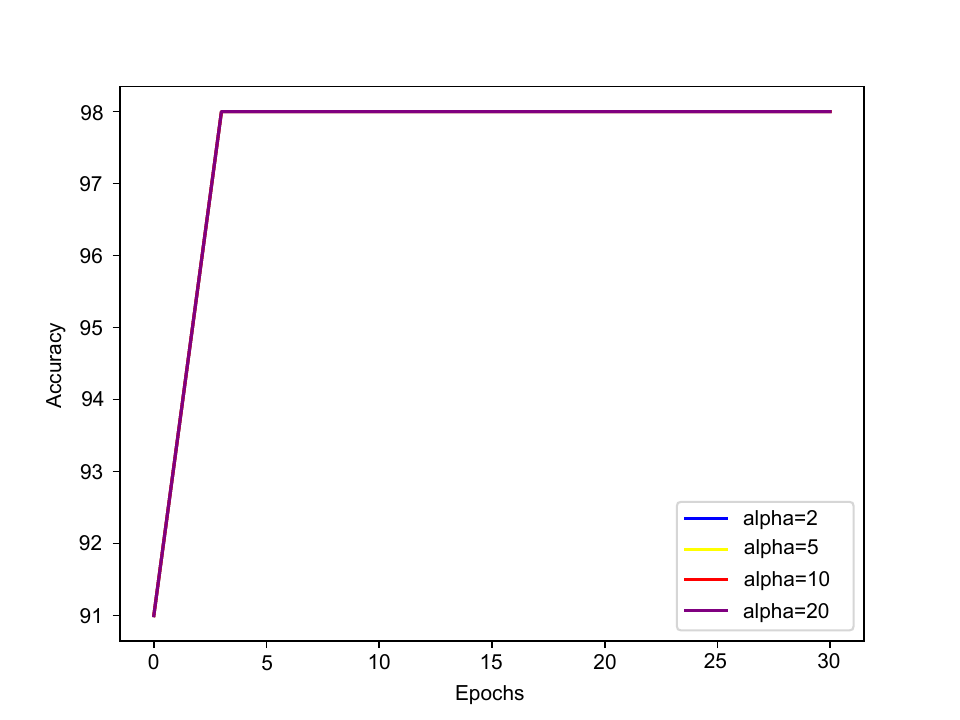}
    \caption{Accuracy of the whole test set along with the repair process with different $\alpha$ settings. All $\alpha$ values can help the compressed network reach $98\%$ accuracy in 3 epochs.}
    \label{fig: acc im}
\end{figure}

To demonstrate the repairing process and indicate the different repair performances with different $\alpha = 2,5,10, 20$, we show the discrepancy reduction and accuracy increase along with the repairing process. The repairing process can reduce the average discrepancy between the two network outputs, as shown in Fig. \ref{fig: dis im}. A small $\alpha$ has a larger discrepancy result because of the larger step size to the optimal value. A large $\alpha$ may not lower the discrepancy to a smaller value. As for the accuracy part, in Fig. \ref{fig: acc im}, all $\alpha$ values can repair the compressed network to reach $98\%$ accuracy, the same as the original network. Thus, our method does mitigate the discrepancy between the original network $\Phi_1$ and the compressed network $\Phi_2$. The $\alpha$ is also important to have a better repair performance, especially for complicated networks.

\section{Conclusions}
This work mainly proposes an approach to repair the compressed FNN based on the equivalence evaluation method. It formally defines the structure of the merged neural network with two given networks and develops reachability analysis methods to compute the reachable set of the discrepancy with the same input. The repair framework is explained in detail, such as the construction of the re-train dataset, the repair result criteria, and the compressed network update. Then, our approach successfully gives repair results between the original and compressed networks by showing the mean discrepancy before and after repair. The repair task is carried out by applying the discrepancy domain to the compressed network output to re-train the compressed network with randomly chosen samples, as shown by the MNIST experiment.

\bibliographystyle{ieeetr}
\bibliography{ref}

\end{document}